\documentclass[12pt,journal]{IEEEtran}

\usepackage{graphicx}
\usepackage{amsmath,amsthm}
\usepackage{cite}
\usepackage{amssymb}
\usepackage{subfigure}
\usepackage{subfigmat}
\usepackage{epstopdf}
\usepackage{xcolor}
\usepackage{algorithm}
\usepackage{algorithmic}
\usepackage{comment}

\usepackage{multirow,multicol,threeparttable}
\usepackage{booktabs}
\usepackage{enumerate}
\usepackage{enumitem}
 
\def\BibTeX{{\rm B\kern-.05em{\sc i\kern-.025em b}\kern-.08em T\kern-.1667em\lower.7ex\hbox{E}\kern-.125emX}}

\newtheorem{theorem}{Theorem}
\newtheorem{corollary}{Corollary}
\newtheorem{lemma}{Lemma}
\newtheorem{fact}{Fact}

\newcommand{\ignore}[1]{}  
\begin{document}

\title{\LARGE \bf Path Planning Algorithms for a Car-Like Robot visiting a set of Waypoints with Field of View Constraints}

\author{Sivakumar Rathinam$^{1}$, Satyanarayana Gupta Manyam$^{2}$, Yuntao Zhang$^3$
\thanks{$^{1}$Associate Professor, Mechanical Engineering, Texas A \& M University, College Station, TX-77843, }%
\thanks{$^{2}$Research Scientist, Infoscitex Corp., a DCS Company, Dayton, OH-45431,}%
\thanks{$^{3}$Graduate Student, Mechanical Engineering, Texas A \& M University, College Station, TX-77843. }%
}
\markboth{}
{Murray and Balemi: Using the Document Class IEEEtran.cls} 


\maketitle
\thispagestyle{empty}
\pagestyle{empty}

\begin{abstract}
This article considers two variants of a shortest path problem for a car-like robot visiting a set of waypoints. The sequence of waypoints to be visited is specified in the first variant while the robot is allowed to visit the waypoints in any sequence in the second variant. Field of view constraints are also placed when the robot arrives at a waypoint, $i.e.$, the orientation of the robot at any waypoint is restricted to belong to a given interval of angles at the waypoint. The shortest path problem is first solved for two waypoints with the field of view constraints using Pontryagin's minimum principle. Using the results for the two point problem, tight lower and upper bounds on the length of the shortest path are developed for visiting $n$ points by relaxing the requirement that the arrival angle must be equal to the departure angle of the robot at each waypoint. Theoretical bounds are also provided on the length of the feasible solutions obtained by the proposed algorithm. Simulation results verify the performance of the bounds for instances with 20 waypoints.
\end{abstract}

\section{Introduction}

Given a set of waypoints to visit and a car-like robot, this article considers two variants of a path planning problem where each waypoint must be visited by the robot and the length of the path traveled by the robot is minimized. Field of view constraints are also placed when the robot arrives at a waypoint, $i.e.$, the orientation of the robot at any waypoint is restricted to belong to a given interval of angles at the waypoint. In the first variant of the problem (Fig. \ref{fig:samplepath}), the sequence in which the waypoints must be visited is specified while the second variant allows for the waypoints to be visited in any sequence. The car-like robot considered here is the Reeds-Shepp vehicle, $i.e.$, a wheeled vehicle that can move forwards and backwards at a constant velocity with a lower bound on its turning radius. 

\begin{figure}[tb!]
\centering{}
\includegraphics[width=\columnwidth]{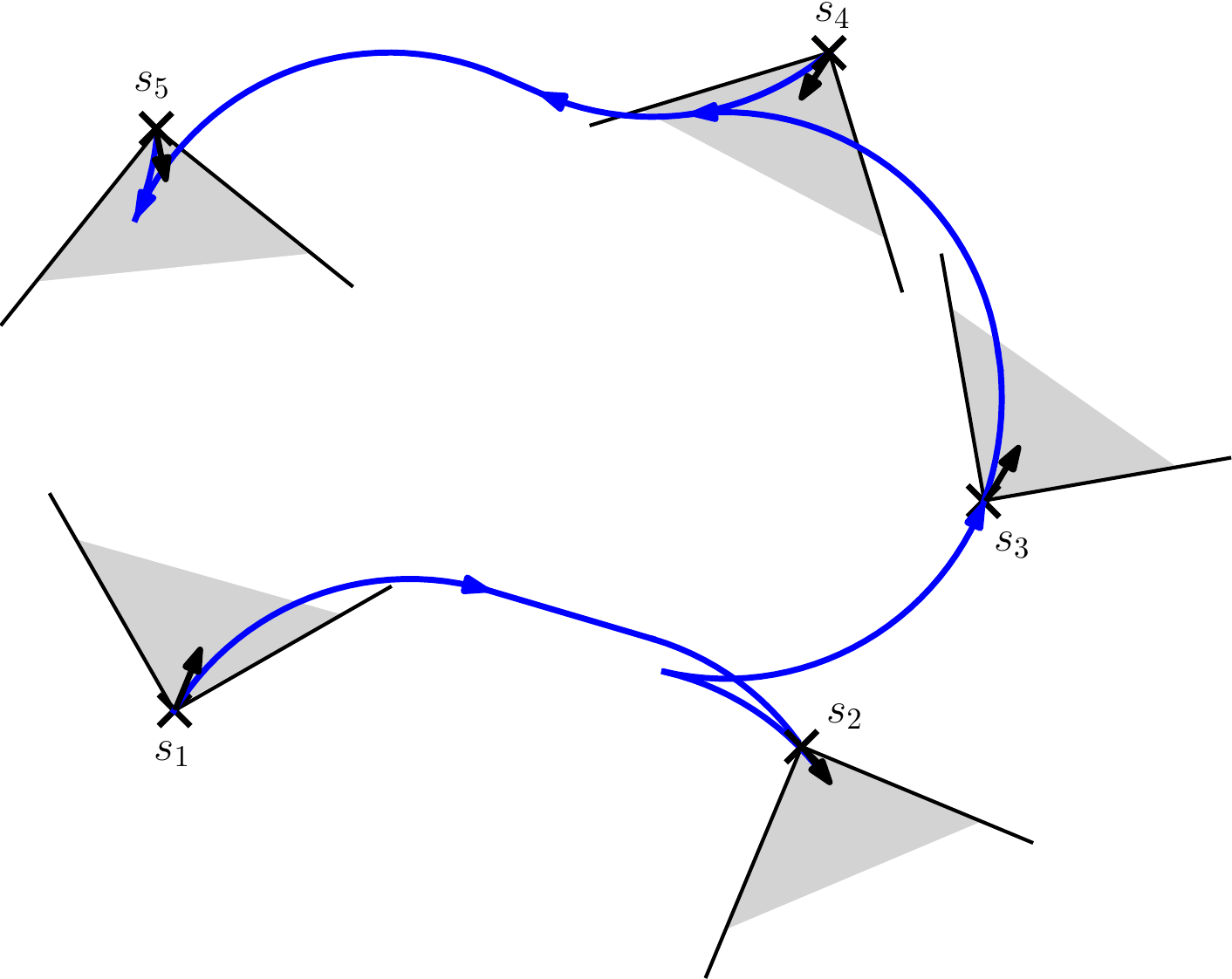}
\caption{A feasible Reeds-Shepp path for an instance with 5 waypoints. The waypoints are visited in the sequence $(s_1,s_2,s_3,s_4,s_5)$. The orientation of the vehicle at any waypoint must belong to the shaded interval of angles specified at the waypoint.}
\label{fig:samplepath}
\end{figure}

This path planning problem is a generalization of the classic point to point, shortest path problem considered by Reeds and Shepp in \cite{reeds1990}. They showed that the shortest path between two oriented points\footnote{A oriented point is a point with a heading angle also specified.} on a 2D plane belongs to a family of 48 paths where each path is a concatenation of at most 5 pieces, each of which is a straight line or an arc of a circle. This shortest path was formulated as an optimal control problem and solved using Pontryagin's minimum principle in Boissonat {\it et al}. \cite{boissonat}, and Sussmann and Tang \cite{Sussmann91shortestpaths}. Sussmann and Tang \cite{Sussmann91shortestpaths} use the theory of envelopes to further reduce the number of paths in the family to 46. Finally, Soueres and Laumond \cite{laumond96} provide a complete synthesis of the shortest paths and add restrictions on the validity of each of the paths in the family. 

The Reeds-Shepp model is closed related to the Dubins model \cite{Dubins1957} where the vehicle is only allowed to move forwards at a constant velocity with a lower bound on its turning radius. The shortest path between two oriented points for the Dubins vehicle belongs to a family of 6 paths where each path is a concatenation of at most 3 pieces, each of which is a straight line or an arc of a circle \cite{Dubins1957,boissonat,Sussmann91shortestpaths}. 

The application of optimal control theory for finding the shortest path for the Reeds-Shepp and Dubins model has lead researchers to obtain shortest path results for several other mobile robots, $i.e.$, refer to the differential drive models in Balkcom and Mason \cite{Balkcom-Mason}, Chitsaz {\it et al.} \cite{chitsaz2006}, the 3D Dubins model in Chitsaz and LaValle \cite{chitsaz-lavalle}, the bounded velocity models in Balkcom and Mason \cite{Balkcom2002}, Furtuna {\it et al.} \cite{furtuna}, the direction dependent Dubins model in Dolinskaya and Maggiar \cite{Dolinkaya}. 

This article is motivated by the generalization of the point to point shortest path problem to include more waypoints and field of view constraints. Note that the angle of visit is not specified at any of the waypoints in both the variants of the path planning problem. As the length of any path between two waypoints in the family of shortest paths is a non-linear function of the angle of visit at the respective waypoints, the path planning problem is non-trivial. In addition, the sequence of waypoints to visit is not specified in the second variant of the path planning problem which already makes it NP-Hard \cite{Applegate} even without the motion constraints, $i.e.$, Euclidean Traveling Salesman Problem (TSP) is its special case. 

The objective of this work is to develop algorithms that can provide solutions to the path planning problem with {\it a-priori} and {\it a-posteriori} guarantees. A-priori guarantees provide a theoretical upper bound on the length of the solutions with respect to the optimum for any instance of the problem. A-posteriori guarantees are obtained by implementing the algorithms on specific instances of the path planning problem and quantifying the deviation in the length of the paths obtained from the optimum.  A-priori guarantees can be numerically much worse compared to the a-posteriori guarantees as they are worst-case bounds valid for {\it any} instance of the problem\cite{TSPbook}.

For the Reeds-Shepp vehicle visiting a set of points with field of view constraints, we are not aware of any algorithms with approximation guarantees. However, in the case of a Dubins vehicle when the sequence of waypoints is specified, Lee et al. \cite{LeeDubins} provide an approximation algorithm with a guarantee\footnote{The guarantee provided by the algorithm here refers to the upper bound on the ratio of the length of the solution produced by the algorithm to the optimum for any instance of the problem.} of 5.03. This result can be further improved to $2+\frac{2}{\pi} + \frac{\pi}{2}\approx 4.21$ using the results by Goaoc {\it et al.} \cite{Goaoc}. 

Bounding the length of a feasible solution obtained by any algorithm with respect to the optimum requires one to either know how to find the optimum or know how to obtain a tight lower bound to the optimum. As we currently do not know how to find an optimal solution for both the variants of the path planning problem, we rely on finding tight lower bounds. In this article, we use the relaxation procedure recently developed in \cite{Manyam2018} for a Dubins vehicle to obtain lower bounds for the Reeds-Shepp vehicle. It is important to note that tightly bounding the length of a feasible solution with respect to a lower bound may still not be trivial for a mobile robot in general. However, since the Reeds-Shepp vehicle is small-time controllable everywhere\footnote{Given any time $t>0$, if the vehicle starts from the origin at time $t=0$, the vehicle can always reach a $\epsilon$-neighborhood of points from the origin with $\epsilon>0$.}\cite{Laumond1998}, we are able to convert lower bounding solutions to feasible solutions with guarantees. 

\subsection{Contributions:} 

\begin{itemize}
\item We first solve the Reeds-Shepp problem for two waypoints with field of view constraints. This problem is referred to as the Reeds-Shepp interval problem and is formulated as an optimal control problem in section \ref{sec:RSI}. A sufficient family of solutions to solve this problem is provided in section \ref{sec:RSfamily} using Pontryagin's minimum principle.
\item For visiting a set of $n$ waypoints, we first use a relaxation procedure to find a tight lower bound to the optimal length of the path planning problem (section \ref{sec:Approx}). In this procedure, we relax the constraint that the arrival angle and the departure angle of the robot at any waypoint must be the same, but restrict the absolute difference between the arrival and departure angles to be less than a given size.
\item If the sequence of waypoints is specified, the lower bounding problem requires finding a shortest path on a directed, acyclic graph. If the sequence of waypoints is not specified, the lower bounding problem requires solving a one-in-a set TSP. The cost of each edge in the graph in both the variants is obtained by solving the Reeds-Shepp interval problem. The solution to the lower bounding problem doesn't yet provide a feasible solution to the path planning problem; therefore, a simple heuristic is used to convert a lower bounding solution to a feasible solution for the path planning problem. 
\item We then provide theoretical bounds on the length of the solutions obtained by the proposed algorithm in section \ref{sec:Approx}.
\item Simulation results verify the guarantees provided by the algorithms for both the variants in section \ref{sec:sim}.
\end{itemize}

\section{Review of Reeds-Shepp's shortest path result}

We use the notation followed in Sussmann and Tang\cite{Sussmann91shortestpaths}, and in Soueres and Laumond \cite{laumond96} to present the sufficient family of paths for the Reeds-Shepp vehicle: $S$ denotes a straight line segment and $C$ an arc of circle of radius $\rho$. Subscripts denote the length of the straight line segments or the angle of turn in the arcs. Left and right turns are represented using $l$ and $r$ respectively. Furthermore, superscripts also specify if the vehicle is moving forward ($+$) or backward ($-$). Cusps (instants when the vehicle changes its velocity from moving forward to backward or vice-versa) are denoted using the symbol``$|$". Note that the symbol $|$ is not used when its usage is redundant, for example, in paths like $l^+l^-$ or $r^+r^-$ etc.  Table \ref{RSpaths} presents the sufficient family of solutions for the Reeds-Shepp problem proved in \cite{laumond96}. 

{{\small
\begin{table}[h!]
\begin{center}
\caption{Sufficient family of solutions for the Reeds-Shepp Problem}
\begin{tabular}{p{.1in}p{1.2in}p{1.6in}}
\hline
No. & Path type & Limitations \\ \hline
1 & $C_a | C_b | C_e$ & $a+b+e \leq \pi$ \\
2 & $C_aS_dC_b $ & $0\leq a,b\leq \frac{\pi}{2}, 0\leq d$   \\
3 & $C_a | C_b C_e$ or $C_e C_b |C_a$  & $0\leq a,e \leq b, 0\leq b\leq \frac{\pi}{2}$  \\
 &  & If $a=b$, $b\leq \frac{\pi}{3}$ \\
4 & $C_a | C_b C_b | C_e$ & $0\leq a,e < b, 0\leq b\leq \frac{\pi}{2}$  \\
5 & $C_a C_b | C_b C_e$ & $0\leq a,e < b, 0\leq b\leq \frac{\pi}{3}$  \\
6 & $C_a|C_{\frac{\pi}{2}} S_d C_{\frac{\pi}{2}} |C_b$ & $0\leq a,b < \frac{\pi}{2}, 0\leq d$  \\
7 &  $C_a|C_{\frac{\pi}{2}} S_d C_{b}$ or$~~~$ $C_bS_dC_{\frac{\pi}{2}} | C_{a}$ & $0\leq a \leq \pi, 0\leq b \leq \frac{\pi}{2}, 0 \leq d$  \\ \hline  \\
\end{tabular}
\label{RSpaths}
\end{center}
\end{table}
}}
  
\section{Path Planning Problem Statement}
The position and the orientation of the vehicle at time $t$ is represented as $(x(t),y(t),\theta(t))$, and $\rho$ denotes the minimum turning radius of the vehicle. Let $S:=\{1,2,\cdots,n\}$ denote the set of waypoints. Waypoint $i\in S$ is located at $(x_i,y_i)$ and must be visited at an angle in the interval $ I_i :=[\theta_i^{min},\theta_i^{max}]$ where $\theta_i^{min} < \theta_i^{max}$. 

Let the vehicle visit waypoint $i$ at orientation $\theta_i$. Given the orientations $\theta_i$ and $\theta_j$ at any two waypoints $i,j\in S$, let $d_{ij}(\theta_i,\theta_j)$ denote the length of the shortest Reeds-Shepp path to travel between $(x_i,y_i,\theta_i)$ and $(x_j,y_j,\theta_j)$.

In the first variant of the problem, let the sequence to visit the waypoints be in the order $(s_1,s_2,s_3,\cdots,s_n)$ where $s_i\in S, i=1,\cdots,n$. In the given sequence, the vehicle is starting at $s_1$ and is visiting $s_2$ next and so on, and ending its path at $s_n$. The objective is to find the orientations $\theta_{s_i}\in I_{s_i}, i=1,\cdots,n$ at the waypoints such that $\sum_{i=1}^{n-1} d_{s_is_{i+1}} (\theta_{s_i},\theta_{s_{i+1}}) $ is minimized.

The second variant of the problem aims to find the sequence also in addition to the orientations. The objective here is to find the sequence $(s_1,s_2,\cdots,s_n)$ in which the waypoints must be visited and the orientations $\theta_{s_i}\in I_{s_i}, i=1,\cdots,n$ such that each waypoint is visited once and $\sum_{i=1}^{n-1} d_{s_is_{i+1}} (\theta_{s_i},\theta_{s_{i+1}}) + d_{s_ns_1} (\theta_{s_{n}},\theta_{s_{1}})$ is minimized. Note that in the second variant, the vehicle returns to its first waypoint after visiting all the remaining waypoints.

To address the path planning problem, we first solve the following shortest path problem for just two points: Find the shortest path for the Reeds-Shepp vehicle from $(x_1,y_1)$ to $(x_2,y_2)$ such that $\theta_1\in I_1$ and $\theta_2 \in I_2$. This problem is referred to as the Reeds-Shepp interval problem. Solutions to this problem will play a crucial role later in developing lower and upper bounds for the path planning problem. 

\section{Reeds-Shepp Interval Problem}\label{sec:RSI}

There are two control inputs to the Reeds-Shepp vehicle: the velocity ($u_1$) and the turn rate ($u_2$) of the vehicle. Denote the input vector as $u = (u_1,u_2)$ where $u_1 \in [-1, 1] $\footnote{Note here that $u_1$ belongs to a closed set of values between -1 and 1 rather than belonging to $\{-1,+1\}$. As shown in Boissonat {\it et al}. \cite{boissonat} and Sussmann and Tang \cite{Sussmann91shortestpaths}, the optimal values of $u_1$ from the closed set $[-1,+1]$ in any case turns out to be -1 or +1. So, the assumption that $u_1 \in [-1,+1]$ is valid.} and $u_2 \in [-\frac{1}{\rho},\frac{1}{\rho}] $. The Reeds-Shepp interval problem is formulated as an optimal control problem as follows:

\begin{align}
\min_{u_1(t) \in [-1,1], u_2(t)\in [-\frac{1}{\rho},\frac{1}{\rho}]} &\int_0^{t_f} 1 dt
\end{align}
subject to
\begin{align}
\frac{dx}{dt} &= u_1\cos \theta, \nonumber \\
\frac{dy}{dt} &= u_1 \sin \theta,  \nonumber \\
\frac{d \theta}{dt} &= u_2,  \label{eq:kinematics}
\end{align}

and the following boundary conditions:
\begin{align}
x(0)=x_1, &\hspace{15pt} x(t_f)=x_2, \\
y(0)=y_1, &\hspace{15pt}  y(t_f)=y_2, \\
\theta_{1}^{min}-&\theta(0) \leq 0, \label{eq:theta0_1}\\
\theta(0) - &\theta_{1}^{max} \leq 0,\\
\theta_{2}^{min} - &\theta(t_f)\leq 0, \\
\theta(t_f) - &\theta_{2}^{max} \leq 0. \label{eq:thetaT_2}
\end{align}

Let the adjoint variables associated with $p(t)=(x(t),y(t),\theta(t))$ be denoted as $(\lambda_x(t),\lambda_y(t),\lambda_{\theta}(t))$. The Hamiltonian associated with the above system is defined as
\begin{align}
H(\Lambda,p,u) = \lambda_o +  u_1 \cos \theta \lambda_x +  u_1 \sin \theta \lambda_y + u_2 \lambda_{\theta} 
\end{align}
where $\Lambda(t)=$ $(\lambda_o, \lambda_x(t),\lambda_y(t),\lambda_{\theta}(t))$ and $\lambda_o$ is a scalar parameter. The differential equations governing the adjoint variables are defined as:
\begin{align}
\frac{d\lambda_x}{dt} &= 0, \label{eq:lambdax} \\
\frac{d\lambda_y}{dt} &= 0, \label{eq:lambday} \\
\frac{d \lambda_{\theta}}{dt} &= u_1 \sin \theta \lambda_x - u_1 \cos \theta \lambda_y. \label{eq:lambdat}
\end{align}

Applying the Pontryagin's minimum principle \cite{pontryagin} to the above problem, we obtain the following: If $u^*$ is an optimal control to the Reeds-Shepp interval problem, then there exists a non-zero adjoint vector $\Lambda(t)$ and $t_f>0$ such that $p(t),\Lambda(t)$ being the solution to the equations in \eqref{eq:kinematics} and \eqref{eq:lambdax}-\eqref{eq:lambdat} for $u(t) = u^*(t)$, the following conditions must be satisfied:

\begin{itemize}
\item  $H(\Lambda,p,u^*) \equiv$
$ \min_{u_1(t) \in [-1,1], u_2(t)\in [-\frac{1}{\rho},\frac{1}{\rho}]}H(\Lambda,p,u)$ $\forall t\in [0,t_f]$. \\

\item $H(\Lambda,p,u^*) \equiv 0$ $\forall t\in [0,t_f]$. \\
\item Suppose $\alpha_1,\alpha_2,\beta_1,\beta_2$ are the Lagrange multipliers corresponding to the boundary conditions in \eqref{eq:theta0_1}-\eqref{eq:thetaT_2} respectively. Then, we have,
    \begin{align}
    \alpha_1,\alpha_2,\beta_1,\beta_2  \geq 0, \label{eq:dual}\\
    \alpha_1(\theta_{1}^{min}-\theta(0)) =  0, \label{eq:alpha1} \\
\alpha_2(\theta(0) - \theta_{1}^{max}) = 0, \label{eq:alpha2}\\
\beta_1(\theta_{2}^{min}-\theta(t_f)) = 0, \label{eq:beta1}\\
\beta_2(\theta(t_f) - \theta_{2}^{max}) = 0, \label{eq:beta2}\\
\lambda_{\theta}(t_f) = \beta_2-\beta_1, \label{eq:lambdaT}\\
\lambda_{\theta}(0) = \alpha_1-\alpha_2. \label{eq:lambda0}
    \end{align}
\end{itemize}

Let $t_f^*$ denote the optimal time when the vehicle reaches $(x_2,y_2)$. We first summarize the main results on the Reeds-Shepp problem given the orientations at the waypoints from \cite{boissonat},\cite{Sussmann91shortestpaths}. These results will be used to solve the interval problem.

\begin{fact}\label{fact1}
Along an optimal path, $u_1$ is either +1 or -1 and $u_2$ is either $\frac{1}{\rho}$ or $-\frac{1}{\rho}$. There are two types of paths possible. The first type is of the form $C|C|C$ where $u_1$ is singular and $u_2$ doesn't change sign. The second type of paths lies between three parallel lines $D^+$, $D^-$ and $D^0$. The straight line segments and the inflection points (where $u_2$ switches sign) occur on $D^0$. All the cusps occur on $D^+$ or $D^-$. The orientation of the paths at any cusp is perpendicular to $D^0$.
\end{fact}

\begin{fact}\label{fact2}
 The first type of paths, $C|C|C$, is possible only when $\lambda_\theta(t) \neq 0$ for all $t\in [0,t_f^*]$.
\end{fact}

\begin{fact}\label{fact3}
Equations \eqref{eq:lambdax}-\eqref{eq:lambday} imply $\lambda_x$ and $\lambda_y$ are constants. Equation \eqref{eq:lambdat} implies $\dot{\lambda}_\theta = \lambda_x \dot{y} - \lambda_y \dot{x}$. Integrating, we get, ${\lambda}_\theta = \lambda_x {y} - \lambda_y{x} + c$. Therefore, all the points on an optimal path with the same values of $\lambda_\theta$ lie on a straight line. 
\end{fact}

\begin{fact}\label{fact4}
Any point along the optimal path corresponding to $\lambda_\theta = 0$ must lie on $D^0$.
\end{fact}


{{\small
\begin{table*}[h!]
\begin{center}
\caption{Validity of Reeds-Shepp Interval solutions for $\theta(0)=\theta_1^{max}$ and $\theta(t_f^*)=\theta_2^{min}$}
\begin{tabular}{p{1.1in}p{.6in}p{.65in}p{.7in}p{.6in}p{.6in}p{.8in}p{.8in}}
\hline
Path type: & $C_a | C_b | C_e$ & $C_aS_dC_b$ &  $C_a | C_b C_e$ or $C_e C_b |C_a$ & $C_a | C_b C_b | C_e$ & $C_a C_b | C_b C_e$ & $C_a|C_{\frac{\pi}{2}} S_d C_{\frac{\pi}{2}} |C_b$ & $C_a|C_{\frac{\pi}{2}} S_d C_{b}$ or$~~~$ $C_bS_dC_{\frac{\pi}{2}} | C_{a}$
 \\ 
 \\
Limitations:  & $a+b+e \leq \pi$ & $0\leq a,b\leq \frac{\pi}{2}$ &  $0\leq a,e \leq b$ & $0\leq a,e < b$ & $0\leq a,e < b$ & $0\leq a,b < \frac{\pi}{2}$ & $0\leq a \leq \pi$
 \\ 
   &  & $0\leq d$ &  $0\leq b\leq \frac{\pi}{2}$ & $0\leq b\leq \frac{\pi}{2}$ & $0\leq b\leq \frac{\pi}{3}$ & $0\leq d$ & $ 0\leq b \leq \frac{\pi}{2}$$~~~~$ $0 \leq d$
 \\ 
    &  & &  If $a=b,~~~~$, $0\leq b\leq \frac{\pi}{3}$ & & & & \\ 
 \hline  \\
Boundary conditions: & & & & & & & \\ 
$\lambda_{\theta}(0) =\lambda_{\theta}(t_f^*)= 0$ &  Not possible  & $S_d$ & $l_b | l_b$, $r_b | r_b$ & Not possible &  $l_b | l_b$, $r_b | r_b$ &  Not possible & $l_\frac{\pi}{2} | l_{\frac{\pi}{2}} S_d$\\ 
& & & & & & & \\
$\lambda_{\theta}(0) < 0,\lambda_{\theta}(t_f^*)= 0$ & Not possible & $l_aS_d$ & $l_a |  l_b, l_ar_b|r_b$ & Not possible & $l_ar_b|r_b$ & Not possible  & $l_a | l_{\frac{\pi}{2}} S_d$ \\
& & & & & & & \\
$\lambda_{\theta}(0) = 0,\lambda_{\theta}(t_f^*)< 0$ & Not possible & $S_dl_b$ & $l_b |  l_a, r_b| r_bl_a$ & Not possible &  $r_b|r_bl_a$ & Not possible & $S_dl_\frac{\pi}{2}|l_{a} $ \\
& & & & & & & \\
$\lambda_{\theta}(0) < 0,\lambda_{\theta}(t_f^*)< 0$ & $l_a|l_b|l_e$ & $l_aS_dl_b$ & $l_a | l_b$ & Not possible & $l_ar_b|r_bl_e$ & $l_a|l_\frac{\pi}{2}S_dl_\frac{\pi}{2}|l_b$ & $l_a|l_{\frac{\pi}{2}} S_d l_b$ $l_b S_dl_{\frac{\pi}{2}} | l_a$\\
\hline \\
\end{tabular}
\label{table3}
\end{center}
\end{table*}
}}

\section{Sufficient family of solutions to the Reeds-Shepp Interval Problem}\label{sec:RSfamily}

We will state the first main result of this article.

\begin{theorem}\label{theorem1} The shortest path for the Reeds-Shepp vehicle between the waypoints $(x_1,y_1)$ and $(x_2,y_2)$ where $\theta(0)\in I_1$ and $\theta(t_f^*) \in I_2$ that is piecewise $C^2$, and either $C^1$ or with cusps, must be of type listed below or a subset thereof. 
\begin{itemize}
\item $\theta(0)=\theta_1^{max}$ and $\theta(t_f^*) = \theta_2^{min}$: \vspace{.2cm} \newline 
$l|l|l$,  $lSl$, $l|l$, $lr|rl$, $l|l_\frac{\pi}{2}Sl_\frac{\pi}{2}|l$, $l|l_{\frac{\pi}{2}} S l$, $l Sl_{\frac{\pi}{2}} | l$.  \vspace{.2cm}
\item $\theta(0)=\theta_1^{min}$ and  $\theta(t_f^*)=\theta_2^{max}$: \vspace{.2cm} \newline 
$r|r|r$,  $rSr$, $r|r$, $rl|lr$, $r|r_\frac{\pi}{2}Sr_\frac{\pi}{2}|r$, $r|r_{\frac{\pi}{2}} S r$, $r Sr_{\frac{\pi}{2}} | r$.  \vspace{.2cm}
\item $\theta(0)=\theta_1^{max}$ and $\theta(t_f^*)=\theta_2^{max}$: \vspace{.2cm} 
\newline
$lSr$, $l | lr$, $lr|r$, $l| lr| r$, $l|l_\frac{\pi}{2}Sr_\frac{\pi}{2}|r$, $l|l_{\frac{\pi}{2}} S r$, $lSr_\frac{\pi}{2}|r$. 
\vspace{.2cm} 

\item $\theta(0)=\theta_1^{min}$ and $\theta(t_f^*)=\theta_2^{min}$: \vspace{.2cm} \newline 
$rSl$, $r| rl$, $rl|l$, $r|  rl| l$, $r|r_\frac{\pi}{2}Sl_\frac{\pi}{2}|l$, $r|r_{\frac{\pi}{2}} S l$, $rSl_\frac{\pi}{2}|l$.
\vspace{.2cm} 

\item $\theta_1^{min} < \theta(0) < \theta_1^{max}$ and $\theta_2^{min} < \theta(t_f^*) < \theta_2^{max}$: \vspace{.2cm} 
\newline  
$S$, $l| l$, $r|  r$, $l_\frac{\pi}{2}|l_{\frac{\pi}{2}} S$.

\vspace{.2cm}

\item $\theta(0)=\theta_1^{min}$ and $ \theta_2^{min} < \theta(t_f^*) < \theta_2^{max}$:  \vspace{.2cm}
\newline  
$rS$, $r|r$, $rl|l$, $r|r_{\frac{\pi}{2}} S$.
\vspace{.2cm}

\item $\theta_1^{min} < \theta(0) < \theta_1^{max}$ and $\theta(t_f^*)=\theta_2^{min} $: \vspace{.2cm} 
\newline 
$Sl$, $l|l$, $r|rl$, $Sl_\frac{\pi}{2}|l$.
\vspace{.2cm}

\item $ \theta_1^{min} < \theta(0) < \theta_1^{max}$ and  $\theta(t_f^*)=\theta_2^{max} $: \vspace{.2cm} 
\newline 
$Sr$, $r| r$, $l|lr$, $Sr_\frac{\pi}{2}|r$.
\vspace{.2cm}

\item $\theta(0)=\theta_1^{max}$ and $\theta_2^{min} < \theta(t_f^*) < \theta_2^{max} $:  \vspace{.2cm}
\newline  
$lS$, $l|l$, $lr|r$, $l|l_{\frac{\pi}{2}} S$.
\vspace{.2cm}

\end{itemize}
\end{theorem}

In the rest of this section, we will prove this theorem. The orientation of the vehicle at waypoint 1 must either belong to the boundaries of $I_1$ or belong to the interior of $I_1$, $i.e.$, $\theta(0) = \theta_1^{min}$ or $\theta(0) = \theta_1^{max}$ or $\theta(0) \in (\theta_1^{min},\theta_1^{max})$. Similarly, $\theta(t_f^*) = \theta_2^{min}$ or $\theta(t_f^*) = \theta_2^{max}$ or $\theta(t_f^*) \in (\theta_2^{min},\theta_2^{max})$. Combinations of choices of these angles influence the $\lambda_\theta$ values at the waypoints through equations \eqref{eq:dual}-\eqref{eq:lambda0} which further restrict the choice of solutions for the interval problem. We will prove the result for $\theta(0)=\theta_1^{max}$ and $\theta(t_f^*) = \theta_2^{min}$. Other combinations of angles in the theorem can be shown in a similar way. 

\begin{lemma}\label{lemma1}
$\theta(0)=\theta_1^{max}$ and $\theta(t_f^*) = \theta_2^{min}$
implies $\lambda_\theta(0) \leq 0$ and $\lambda_\theta(t_f^*) \leq 0$.   
\end{lemma}
\begin{proof}
Using equation \eqref{eq:alpha1}, $\theta(0)=\theta_1^{max}$ implies $\alpha_1 = 0$. Therefore, $\lambda_\theta(0) = \alpha_1-\alpha_2 \leq 0$. Using equation \eqref{eq:beta2},  $\theta(t_f^*)=\theta_2^{min}$ implies $\beta_2 = 0$. Therefore, $\lambda_\theta(t_f^*) = \beta_2-\beta_1 \leq 0$. 
\end{proof}

\begin{lemma}\label{lemma2}
$\lambda_\theta(0) = 0$ and $\lambda_\theta(t_f^*) = 0$ allows for the following set of solutions for the Reeds-Shepp interval problem: $S$, $l| l$, $r|  r$, $l_\frac{\pi}{2}|l_{\frac{\pi}{2}} S$. 
\end{lemma}

\begin{proof} 
Using Fact \ref{fact2}, $C|C|C$ is not possible. $\lambda_\theta(0) =\lambda_\theta(t_f^*) = 0$ implies that both the end points of an optimal path must lie on $D^0$ (Fact \ref{fact4}). This will further limit the validity of each of the paths given in table \ref{RSpaths} as follows:
\begin{itemize}
\item $C_aS_dC_b$ is not possible unless the turn angles $a=b=0$. 
\item $C_a|C_b C_e$ or $C_eC_b | C_a$ is not possible unless $e=0$ and $a=b$. Similarly, $C_aC_b|C_b C_e$ is not possible unless $a=e=0$. 
\item $C_a|C_{\frac{\pi}{2}} S_d C_{\frac{\pi}{2}} |C_b$ is not possible because the end points cannot lie on $D^0$.
\item $C_a|C_{\frac{\pi}{2}} S_d C_{b}$ is possible provided $a=\frac{\pi}{2}$ and $b=0$. In this type, it is sufficient to consider $l^+_\frac{\pi}{2}l^-_{\frac{\pi}{2}} S$ and $l^-_\frac{\pi}{2}l^+_{\frac{\pi}{2}} S$ as the other possibilities $r^+_\frac{\pi}{2}r^-_{\frac{\pi}{2}} S$ and $r^-_\frac{\pi}{2}r^+_{\frac{\pi}{2}} S$ produce the same lengths and angles at the waypoints.
\end{itemize}
Therefore, for the given boundary conditions, a sufficient family of solutions is \{$S$, $l|l$, $r|r$, $l_\frac{\pi}{2}|l_{\frac{\pi}{2}} S$\}.
\end{proof}

\begin{lemma}
$\theta(0)=\theta_1^{max}$ and $\theta(t_f^*) = \theta_2^{min}$ allows for the following set of solutions for the Reeds-Shepp interval problem or a subset thereof: $l|l|l$,  $lSl$, $l|l$, $lr|rl$, $l|l_\frac{\pi}{2}Sl_\frac{\pi}{2}|l$, $l|l_{\frac{\pi}{2}} S l$, $l Sl_{\frac{\pi}{2}} | l$.   
\end{lemma}
\begin{proof}
Lemma \ref{lemma1} implies $\lambda_\theta(0) \leq 0$ and $\lambda_\theta(t_f^*) \leq 0$. Therefore, there are four possibilities for the adjoint variable $\lambda_\theta$ at the end points: \begin{enumerate}

\item $\lambda_\theta(0) = 0, \lambda_\theta(t_f^*) =0$,
\item $\lambda_\theta(0) < 0,\lambda_\theta(t_f^*) =0$,
\item  $\lambda_\theta(0) = 0,\lambda_\theta(t_f^*) <0$, or, 
\item  $\lambda_\theta(0) < 0,\lambda_\theta(t_f^*) <0$.
\end{enumerate} Solutions corresponding to each of the possibilities are listed in table \ref{table3} following the same procedure in Lemma \ref{lemma2}. Combining all the solutions in the table proves the Lemma.
\end{proof}

\section{Algorithm $Approx$ for solving the Path Planning Problem}\label{sec:Approx}
$Approx$ is explained first when the sequence of waypoints is specified. We will then modify $Approx$ to address the second variant. $Approx$ first finds a lower bound by relaxing the requirement that the arrival angle of the vehicle at any waypoint must be equal to the departure angle of the vehicle at the waypoint. It then converts the lower bounding solution to a feasible solution using a shortest path algorithm. Specifically, the following are the steps in $Approx$. 

\begin{enumerate}
\item Partition the available set of orientations at each waypoint into $k$ sectors or intervals of equal size. This step partitions $I_i$ into the set $\mathcal{I}_i:=\{[\phi_{i0},\phi_{i1}], [\phi_{i1},\phi_{i2}],\cdots,  [\phi_{i(k-1)},\phi_{ik}]\}$ where $\phi_{ij}:= \theta_i^{min} + \frac{j}{k}(\theta_i^{max}-\theta_i^{min})$ for $j=0,\cdots,k$. 
\item Form an acyclic graph $G$ with $nk$ nodes. The node $v_{im}$ in $G$ represents the waypoint $s_i$ and the interval $[\phi_{s_i(m-1)},\phi_{s_im}]\in \mathcal{I}_{s_i}$ for $i=1,\cdots,n$, $m=1,\cdots,k$. An edge is present in $G$ if and only if it connects adjacent waypoints in the given sequence and any of its corresponding intervals. Let $cost(v_{im},v_{(i+1)l})$ denote the length of the shortest Reeds-Shepp interval path between waypoints $s_i$ and $s_{i+1}$ such that $\theta_{s_i}\in [\phi_{s_i(m-1)},\phi_{s_im}]$ and $\theta_{s_{i+1}}\in [\phi_{s_{i+1}(l-1)},\phi_{s_{i+1}l}]$. These lengths can be obtained from the result in Theorem \ref{theorem1}. 
\item {\it Compute a lower bounding solution:} Use Dijkstra's shortest path algorithm to find a path of minimum length from any node in $\{v_{1m}:m=1,\cdots,k\}$ to any node in $\{v_{nm}:m=1,\cdots,k\}$. Suppose an optimal solution is denoted by the sequence of nodes $(v_{1m^*_1},v_{2m^*_2},\cdots,v_{nm^*_n})$. Let the length of this solution be $Cost_{lb} := \sum_{i=1}^{n-1} cost(v_{im_i^*},v_{(i+1)m_{i+1}^*})$. 
\item {\it Compute a feasible (upper bounding) solution:} Let $\theta_{s_id}$ and $\theta_{s_{i+1}a}$ denote the departure and the arrival angles obtained from solving the Reeds-Shepp interval problem from $v_{im_i^*}$ to $v_{(i+1)m_{i+1}^*}$. It is possible that $\theta_{s_ia}\neq \theta_{s_id}$ for any $i=2,\cdots,n-1$. Dijkstra's shortest path algorithm is again used to choose exactly one of the two possible angles ($\theta_{s_ia}$ or $\theta_{s_id}$ for $i=2,\cdots,n-1$) at waypoints $s_2,\cdots,s_{n-1}$ such that the length of the resulting path is minimized. This provides a feasible solution to the path planning problem. 
\end{enumerate}

\begin{lemma}\label{lemmalb}
For any $k\geq 1$, the length of the path computed in step (3) of $Approx$ is a lower bound to the optimal length of the path planning problem when the sequence of waypoints to visit is given. 
\end{lemma}
\begin{proof}
Let the orientation of the vehicle corresponding to an optimal solution at waypoint $s_i$ be denoted as $\theta_{s_i}^o$. Also, let the length of the segment of the optimal solution from $s_i$ to $s_{i+1}$ be denoted as $l^o(s_i,s_{i+1})$. Let $\theta_{s_i}^o$ belong to the interval $[\phi_{s_im_i^o-1},\phi_{s_im_i^o}]\in \mathcal{I}_{s_i}$ for $i=1,\cdots,n$. Then,
\begin{align}
Cost_{lb} &=\sum_{i=1}^{n-1} cost(v_{im_i^*},v_{(i+1)m_{i+1}^*}) \nonumber \\ & \leq \sum_{i=1}^{n-1} cost(v_{im_i^o},v_{(i+1)m_{i+1}^o}) \leq \sum_{i=1}^{n-1} l^o(s_i,s_{i+1}). \nonumber
\end{align}
\end{proof}

\begin{theorem}\label{theorem2}
The length of the feasible solution obtained by $Approx$ is at most equal to $Cost_{lb} + 3\rho \sum_{i=2}^{n-2} \frac{\theta_{s_i}^{max}-\theta_{s_i}^{min}}{k}$.
\end{theorem}
\begin{proof}
As explained in step (4) of Approx, the arrival and departure angles obtained from the lower bounding solution at waypoint $s_i$ may not be the same. The absolute difference between these two angles can at most be equal to $\Delta \theta_i := \frac{\theta_{s_i}^{max}-\theta_{s_i}^{min}}{k}$. Therefore, the length of the feasible solution obtained by $Approx$ must be upper bounded by $Cost_{lb} + \sum_{i=2}^{n-1} dist_{rs}(\Delta \theta_i)$ where $dist_{rs}$ denotes the length of the shortest Reeds-Shepp path between $(0,0,\Delta \theta_i)$ and $(0,0,0)$. In the appendix, we prove that $dist_{rs} (\Delta \theta_i)\leq \frac{3\rho\Delta \theta_i}{2}$. Hence proved.
\end{proof}

\subsection{Modification to $Approx$ to address the second variant}
In step (2) of $Approx$, instead of computing the Reeds-Shepp interval lengths only between any two adjacent waypoints and their corresponding intervals, we need to compute $cost(v_{im},v_{jl})$ for all $i,j=1,\cdots,n, i\neq j$, and $m,l\in 1,\cdots,k$. The lower bounding problem (step (3)) would aim to find a sequence of waypoints to visit $(s_1,\cdots,s_n)$ and choose exactly one interval at each waypoint (denote the chosen interval at $s_i$ to be $[\phi_{s_i(m_i-1)},\phi_{s_im_i}]$ for $i=1,\cdots,n$) such that $\sum_{i=1}^{n-1} cost(v_{im_i},v_{im_{i+1}}) + cost(v_{nm_n},v_{1m_{1}})$ is minimized. This is a One-in-a-set TSP which can be transformed into a single TSP\cite{TSPbook} and solved using CONCORDE \cite{Applegate:2007} software. The lower bounding solution provides a sequence of waypoints to visit but the arrival and departure angles at each waypoint may not be equal; to make it feasible, Dijkstra's algorithm is used to choose one of the two angles at each of the waypoints (Step 4). The following upper bound on the length of the feasible solution obtained can be proved using exactly the same procedure outlined in Lemma \ref{lemmalb} and Theorem \ref{theorem2}. 

\begin{corollary}
The length of the feasible solution obtained by the modified $Approx$ is at most equal to $Cost^{tsp}_{lb} + 3\rho \sum_{i=1}^{n} \frac{\theta_{s_i}^{max}-\theta_{s_i}^{min}}{k}$ where $Cost^{tsp}_{lb}$ denotes the length of the lower bounding solution to the problem when the sequence of waypoints is not specified.
\end{corollary}

\section{Simulation Results}\label{sec:sim}

Twenty five problem instances were generated with each instance having 20 waypoints. The points were sampled from an area of size $1000 \times 1000$ units. The minimum turning radius of the Reeds-Shepp vehicle was assumed to be $1/10^{th}$ of the side length of the area, $i.e.$, 100 units. For each waypoint, the vehicle was restricted to visit the waypoint within a field of view of $\frac{\pi}{2}$ radians. For $i=1,\cdots,n$, $\theta_i^{min}$ was randomly chosen in the interval $[0,\frac{3\pi}{2}]$ and $\theta_i^{max}$ was set to $\theta_i^{min}+\frac{\pi}{2}$. All the algorithms were coded in MATLAB and the computations
were run on MacBook Pro (Intel Core i7 Processor @2.8 GHz, 16 GB RAM).

For the first variant of the path planning problem, a sequence for each instance is obtained by solving the corresponding Euclidean TSP. $Approx$ uses this sequence as an input to find a feasible solution. The deviation of this solution (in \%) from the lower bound for each instance is $100 \times \frac{Cost_{f}-Cost_{lb}}{Cost_{lb}}$ where $Cost_f$ is the length of the  feasible solution produced by $Approx$. The a-priori (or theoretical) deviation of the upper bound (in \%) from the lower bound can also be computed using Theorem \ref{theorem2} as $100 \times \frac{3\rho}{Cost_{lb}} \sum_{i=2}^{n-2} \frac{\Delta\theta_i}{k}$. All the deviations and bounds for the 25 instances for different levels of discretization at each waypoint is shown in Table \ref{tab:RSSeq}. Overall, for $k=16$, the average deviation of all the feasible solutions obtained by $Approx$ with respect to the lower bounds is 1\%. Note that the theoretical deviations are mostly 10 times larger than the actual deviations of the solutions obtained by $Approx$. This is generally observed with a-priori theoretical guarantees \cite{TSPbook} which specify the worst-case bounds for any instance of the problem. 

Using the modified version of $Approx$, feasible solutions and its corresponding deviation from the lower bounding solutions can also be obtained for the second variant, and are shown in Table \ref{tab:RSTSP}. There was no noticeable difference in the deviations of the feasible solutions for both the variants of the problem.  Computationally, the main difference between solving the two variants lies in step (3) of $Approx$ where a One-in-a-set TSP is solved when the sequence is not given as compared to solving a shortest path problem when the sequence is given. The average computation times needed for solving the 25 instances with $k=4,8$ and $16$ were approximately 44, 182 and 743 seconds respectively when the sequence for the way-points are given, and 48, 205 and 1016 seconds respectively for the variant with the sequence of way points not specified. Sample lower bounding and feasible solutions are shown for the two variants of the problem in Figs. \ref{fig:RSSeq} and \ref{fig:RSTSP}.  

\begin{table*}[htbp]
  \centering
  \caption{Results obtained for the Path Planning problem when the sequence of waypoints to be visited is given}
    \begin{tabular}{ccccccccccccc}
    \toprule
    \multicolumn{1}{c|}{\textbf{Instance}} & \multicolumn{3}{c|}{\textbf{Lower Bounds}} & \multicolumn{3}{c|}{\textbf{Upper Bounds}} & \multicolumn{3}{c|}{\textbf{\%Deviation}} & \multicolumn{3}{c}{\textbf{\% (Theoretical) Deviation}} \\
    \multicolumn{1}{c|}{\textbf{\#}} & \multicolumn{1}{c|}{\textit{k=4}} & \multicolumn{1}{c|}{\textit{k=8}} & \multicolumn{1}{c|}{\textit{k=16}} & \multicolumn{1}{c|}{\textit{k=4}} & \multicolumn{1}{c|}{\textit{k=8}} & \multicolumn{1}{c|}{\textit{k=16}} & \multicolumn{1}{c|}{\textit{k=4}} & \multicolumn{1}{c|}{\textit{k=8}} & \multicolumn{1}{c|}{\textit{k=16}} & \multicolumn{1}{c|}{\textit{k=4}} & \multicolumn{1}{c|}{\textit{k=8}} & \textit{k=16} \\
    \midrule
    1     & 3804.79 & 3831.99 & 3851.74 & 3884.38 & 3881.98 & 3879.56 & 2.09  & 1.30  & 0.72  & 55.73 & 27.67 & 13.76 \\
    2     & 4299.66 & 4325.71 & 4344.54 & 4397.87 & 4372.50 & 4369.62 & 2.28  & 1.08  & 0.58  & 49.32 & 24.51 & 12.20 \\
    3     & 4864.10 & 4910.81 & 4939.48 & 4995.13 & 4977.18 & 4974.11 & 2.69  & 1.35  & 0.70  & 43.60 & 21.59 & 10.73 \\
    4     & 4300.68 & 4352.07 & 4391.90 & 4450.17 & 4448.82 & 4436.74 & 3.48  & 2.22  & 1.02  & 49.31 & 24.36 & 12.07 \\
    5     & 4133.21 & 4160.56 & 4175.54 & 4228.62 & 4204.56 & 4203.97 & 2.31  & 1.06  & 0.68  & 51.31 & 25.48 & 12.70 \\
    6     & 4487.35 & 4523.70 & 4547.33 & 4672.95 & 4621.83 & 4580.61 & 4.14  & 2.17  & 0.73  & 47.26 & 23.44 & 11.66 \\
    7     & 4147.91 & 4181.98 & 4200.59 & 4241.13 & 4246.69 & 4233.13 & 2.25  & 1.55  & 0.77  & 51.12 & 25.35 & 12.62 \\
    8     & 4062.19 & 4106.27 & 4128.25 & 4173.11 & 4171.89 & 4158.41 & 2.73  & 1.60  & 0.73  & 52.20 & 25.82 & 12.84 \\
    9     & 4505.07 & 4562.91 & 4597.80 & 4658.69 & 4640.76 & 4637.82 & 3.41  & 1.71  & 0.87  & 47.07 & 23.24 & 11.53 \\
    10    & 4429.36 & 4514.41 & 4558.47 & 4647.27 & 4615.35 & 4610.73 & 4.92  & 2.24  & 1.15  & 47.88 & 23.49 & 11.63 \\
    11    & 4027.89 & 4064.10 & 4079.05 & 4200.24 & 4127.74 & 4109.55 & 4.28  & 1.57  & 0.75  & 52.65 & 26.09 & 13.00 \\
    12    & 4037.90 & 4112.40 & 4156.99 & 4252.12 & 4216.20 & 4210.47 & 5.31  & 2.52  & 1.29  & 52.52 & 25.78 & 12.75 \\
    13    & 4173.03 & 4215.56 & 4239.39 & 4319.83 & 4301.63 & 4281.75 & 3.52  & 2.04  & 1.00  & 50.82 & 25.15 & 12.51 \\
    14    & 3801.39 & 3889.95 & 3933.13 & 4027.25 & 4012.58 & 3994.65 & 5.94  & 3.15  & 1.56  & 55.78 & 27.26 & 13.48 \\
    15    & 3839.41 & 3887.76 & 3920.50 & 3973.89 & 3964.84 & 3959.78 & 3.50  & 1.98  & 1.00  & 55.23 & 27.27 & 13.52 \\
    16    & 4633.02 & 4663.19 & 4675.75 & 4717.49 & 4696.20 & 4693.85 & 1.82  & 0.71  & 0.39  & 45.77 & 22.74 & 11.34 \\
    17    & 3806.29 & 3892.06 & 3953.78 & 4047.83 & 4029.06 & 4021.78 & 6.35  & 3.52  & 1.72  & 55.71 & 27.24 & 13.41 \\
    18    & 4290.46 & 4325.27 & 4351.14 & 4418.13 & 4397.81 & 4396.65 & 2.98  & 1.68  & 1.05  & 49.43 & 24.51 & 12.18 \\
    19    & 4790.18 & 4804.60 & 4810.70 & 4824.55 & 4819.17 & 4817.41 & 0.72  & 0.30  & 0.14  & 44.27 & 22.07 & 11.02 \\
    20    & 4145.96 & 4161.25 & 4167.01 & 4190.37 & 4183.32 & 4178.90 & 1.07  & 0.53  & 0.29  & 51.15 & 25.48 & 12.72 \\
    21    & 3773.39 & 3830.21 & 3849.54 & 3924.78 & 3881.04 & 3882.57 & 4.01  & 1.33  & 0.86  & 56.20 & 27.68 & 13.77 \\
    22    & 4033.64 & 4087.62 & 4117.23 & 4203.93 & 4171.91 & 4164.08 & 4.22  & 2.06  & 1.14  & 52.57 & 25.94 & 12.88 \\
    23    & 4117.51 & 4180.07 & 4213.85 & 4277.71 & 4260.20 & 4257.21 & 3.89  & 1.92  & 1.03  & 51.50 & 25.37 & 12.58 \\
    24    & 4748.48 & 4796.09 & 4834.73 & 4892.20 & 4888.68 & 4884.99 & 3.03  & 1.93  & 1.04  & 44.66 & 22.11 & 10.97 \\
    25    & 4612.48 & 4639.93 & 4659.99 & 4704.55 & 4692.71 & 4688.99 & 2.00  & 1.14  & 0.62  & 45.97 & 22.85 & 11.38 \\
    \bottomrule
    \end{tabular}%
  \label{tab:RSSeq}%
\end{table*}%

{{\small
\begin{table*}[htbp]
  \centering
  \caption{Results obtained for the Path Planning problem when the sequence of waypoints to be visited is not given}
    \begin{tabular}{ccccccccccccc}
    \toprule
    \multicolumn{1}{c|}{\textbf{Instance}} & \multicolumn{3}{c|}{\textbf{Lower Bounds}} & \multicolumn{3}{c|}{\textbf{Upper Bounds}} & \multicolumn{3}{c|}{\textbf{\% Deviation}} & \multicolumn{3}{c}{\textbf{\% (Theoretical) Deviation}} \\
    \multicolumn{1}{c|}{\textbf{\#}} & \multicolumn{1}{c|}{\textit{k=4}} & \multicolumn{1}{c|}{\textit{k=8}} & \multicolumn{1}{c|}{\textit{k=16}} & \multicolumn{1}{c|}{\textit{k=4}} & \multicolumn{1}{c|}{\textit{k=8}} & \multicolumn{1}{c|}{\textit{k=16}} & \multicolumn{1}{c|}{\textit{k=4}} & \multicolumn{1}{c|}{\textit{k=8}} & \multicolumn{1}{c|}{\textit{k=16}} & \multicolumn{1}{c|}{\textit{k=4}} & \multicolumn{1}{c|}{\textit{k=8}} & \textit{k=16} \\
    \midrule
    1     & 4057.72 & 4092.57 & 4113.42 & 4170.19 & 4151.74 & 4144.99 & 2.77  & 1.45  & 0.77  & 55.16 & 27.35 & 13.60 \\
    2     & 4246.45 & 4321.11 & 4364.33 & 4436.61 & 4422.07 & 4418.51 & 4.48  & 2.34  & 1.24  & 52.71 & 25.90 & 12.82 \\
    3     & 4845.76 & 4911.56 & 4959.93 & 5025.39 & 5018.87 & 5016.32 & 3.71  & 2.18  & 1.14  & 46.19 & 22.79 & 11.28 \\
    4     & 4313.29 & 4394.21 & 4455.24 & 4559.68 & 4544.06 & 4540.38 & 5.71  & 3.41  & 1.91  & 51.90 & 25.47 & 12.56 \\
    5     & 4194.38 & 4250.03 & 4282.38 & 4366.42 & 4348.02 & 4325.92 & 4.10  & 2.31  & 1.02  & 53.37 & 26.33 & 13.07 \\
    6     & 4485.86 & 4539.59 & 4574.05 & 4660.59 & 4637.82 & 4598.73 & 3.90  & 2.16  & 0.54  & 49.90 & 24.65 & 12.23 \\
    7     & 4930.22 & 4983.02 & 5009.58 & 5088.47 & 5059.55 & 5045.38 & 3.21  & 1.54  & 0.71  & 45.40 & 22.46 & 11.17 \\
    8     & 4235.85 & 4309.33 & 4340.73 & 4414.70 & 4395.55 & 4378.21 & 4.22  & 2.00  & 0.86  & 52.84 & 25.97 & 12.89 \\
    9     & 4732.49 & 4791.19 & 4810.31 & 4886.03 & 4872.67 & 4834.40 & 3.24  & 1.70  & 0.50  & 47.30 & 23.36 & 11.63 \\
    10    & 4672.81 & 4718.06 & 4737.27 & 4871.78 & 4778.24 & 4768.97 & 4.26  & 1.28  & 0.67  & 47.90 & 23.72 & 11.81 \\
    11    & 4152.42 & 4203.99 & 4237.42 & 4331.90 & 4293.83 & 4285.10 & 4.32  & 2.14  & 1.13  & 53.91 & 26.62 & 13.21 \\
    12    & 4072.57 & 4149.76 & 4186.21 & 4294.42 & 4262.67 & 4247.17 & 5.45  & 2.72  & 1.46  & 54.96 & 26.97 & 13.37 \\
    13    & 4289.76 & 4309.93 & 4324.91 & 4401.09 & 4360.42 & 4354.72 & 2.60  & 1.17  & 0.69  & 52.18 & 25.97 & 12.94 \\
    14    & 3723.66 & 3791.31 & 3832.73 & 3930.67 & 3903.63 & 3892.84 & 5.56  & 2.96  & 1.57  & 60.11 & 29.52 & 14.60 \\
    15    & 3793.96 & 3823.38 & 3846.70 & 3934.15 & 3891.93 & 3882.75 & 3.70  & 1.79  & 0.94  & 59.00 & 29.27 & 14.55 \\
    16    & 4779.49 & 4808.88 & 4827.04 & 4875.53 & 4852.86 & 4851.70 & 2.01  & 0.91  & 0.51  & 46.83 & 23.27 & 11.59 \\
    17    & 3947.45 & 4064.96 & 4130.38 & 4247.01 & 4219.57 & 4205.13 & 7.59  & 3.80  & 1.81  & 56.70 & 27.53 & 13.55 \\
    18    & 4503.20 & 4528.22 & 4551.88 & 4609.50 & 4595.95 & 4590.49 & 2.36  & 1.50  & 0.85  & 49.71 & 24.72 & 12.29 \\
    19    & 4690.94 & 4751.28 & 4786.94 & 4879.54 & 4832.68 & 4828.24 & 4.02  & 1.71  & 0.86  & 47.72 & 23.56 & 11.69 \\
    20    & 4210.87 & 4244.68 & 4265.52 & 4315.18 & 4301.56 & 4299.93 & 2.48  & 1.34  & 0.81  & 53.16 & 26.37 & 13.12 \\
    21    & 4261.76 & 4316.41 & 4347.47 & 4405.77 & 4397.88 & 4389.17 & 3.38  & 1.89  & 0.96  & 52.52 & 25.93 & 12.87 \\
    22    & 4188.74 & 4227.72 & 4250.21 & 4362.53 & 4294.34 & 4287.60 & 4.15  & 1.58  & 0.88  & 53.44 & 26.47 & 13.17 \\
    23    & 4326.04 & 4389.36 & 4424.83 & 4486.75 & 4467.78 & 4468.90 & 3.71  & 1.79  & 1.00  & 51.74 & 25.50 & 12.65 \\
    24    & 4846.61 & 4889.53 & 4931.16 & 4989.72 & 4984.38 & 4983.70 & 2.95  & 1.94  & 1.07  & 46.18 & 22.89 & 11.35 \\
    25    & 4626.99 & 4672.73 & 4688.99 & 4777.01 & 4751.44 & 4709.62 & 3.24  & 1.68  & 0.44  & 48.38 & 23.95 & 11.93 \\
    \bottomrule
    \end{tabular}%
  \label{tab:RSTSP}%
\end{table*}%
}}

\begin{figure}[htb!]
\centering{}
\includegraphics[width=0.9\columnwidth]{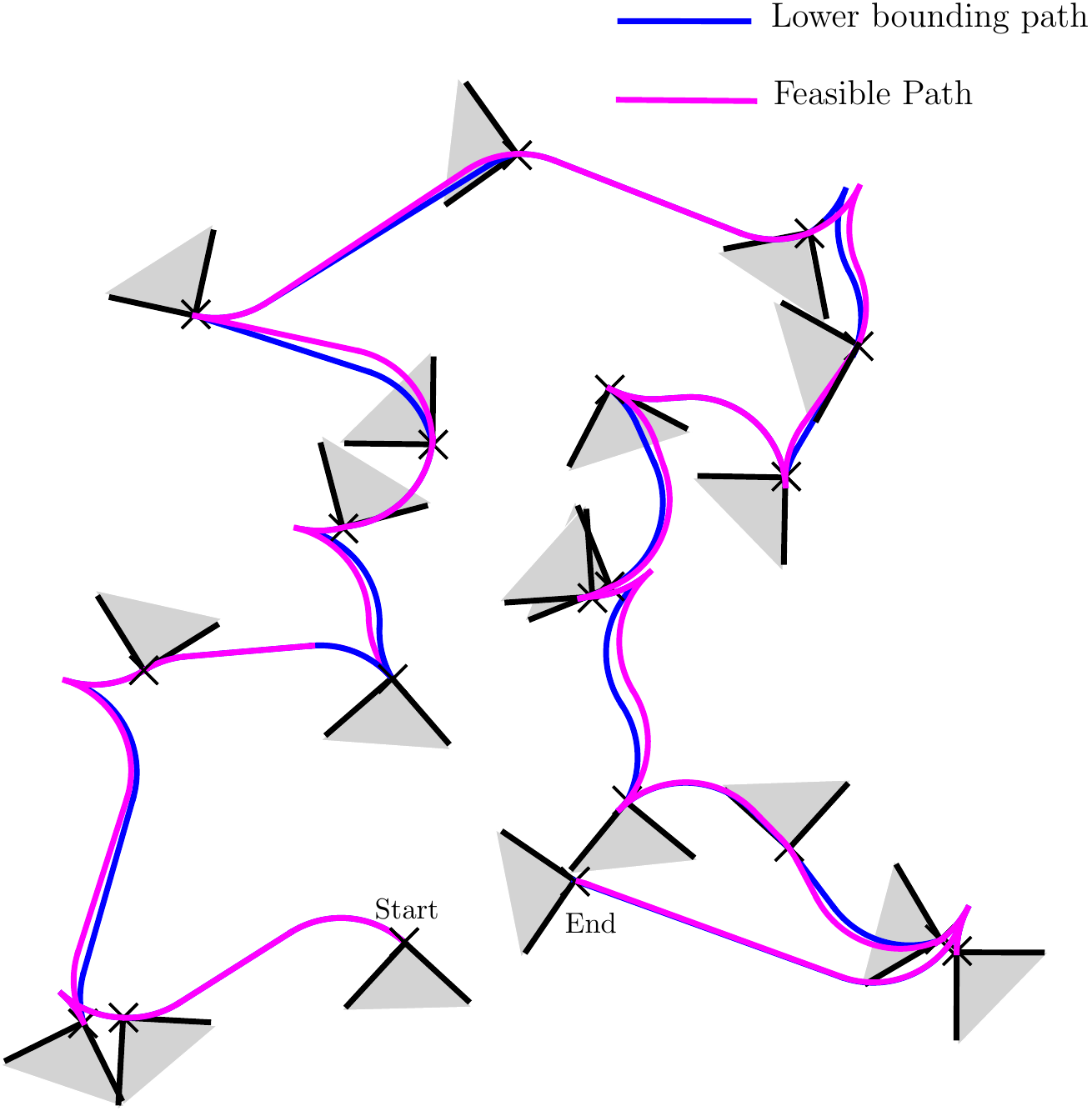}
\caption{Solutions when the sequence of waypoints to visit is given.}
\label{fig:RSSeq}
\end{figure}

\begin{figure}[htb!]
\centering{}
\includegraphics[width=0.9\columnwidth]{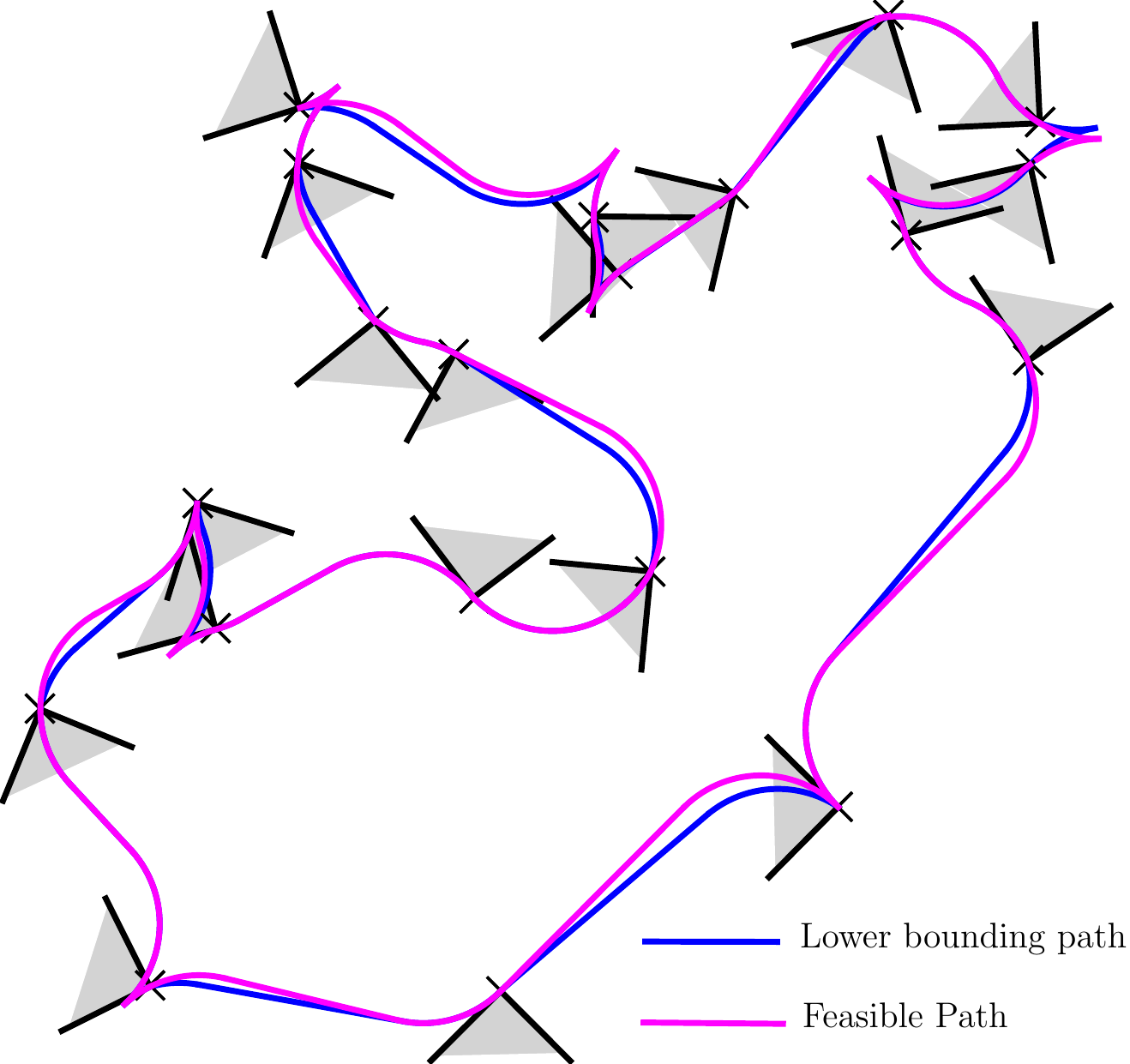}
\caption{Solutions when the sequence to visit is not specified.}
\label{fig:RSTSP}
\end{figure}

\section{Conclusions}
This article presented a path planning algorithm for a car-like robot visiting a set of waypoints with field of view constraints. The problem is first solved for two waypoints and the solutions to the two-point problem are generalized to handle multiple waypoints using a relaxation procedure. Theoretical bounds are also obtained on the quality of the feasible solutions produced by the proposed algorithm. The procedure used in this article is generic and can be applied to any mobile robot as long as one can solve the point to point problem for the robot. Future work can address problems with multiple vehicles, fuel constraints on vehicles, uncertainty in the location of waypoints and motion in 3D. 

\bibliographystyle{IEEEtran}
\bibliography{car}

\section*{Appendix}
\begin{figure}[htb!]
\centering{}
\includegraphics[width=0.6\columnwidth]{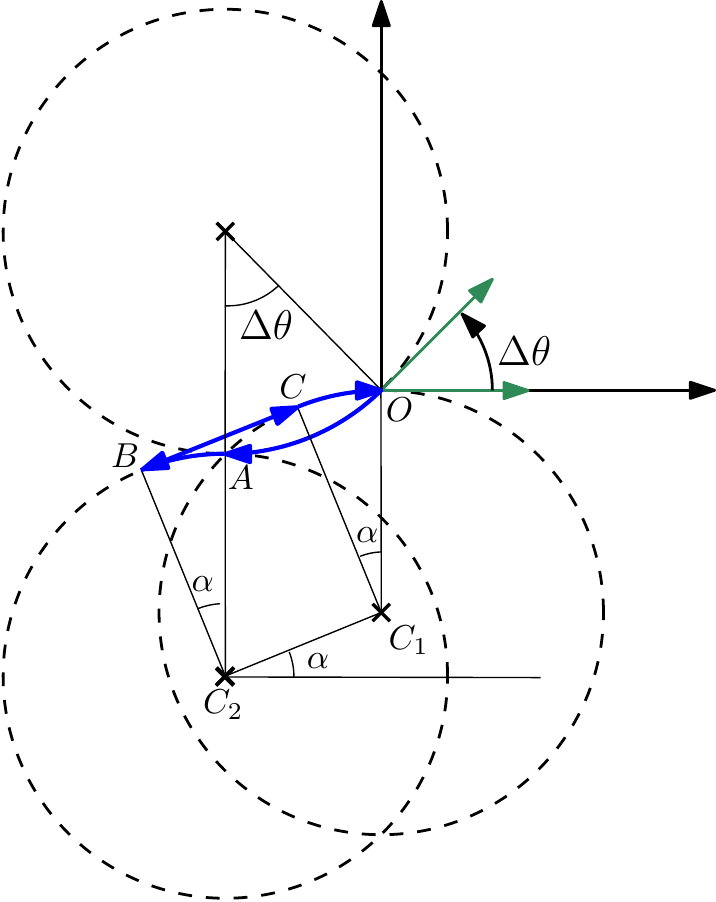}
\caption{Construction of Reeds-Shepp path from $(0,0,\Delta \theta)$ to $(0,0,0)$}
\label{fig:3rhotheta}
\end{figure}
The following lemma is a simplification of a more general bounding result proved for the Reeds-Shepp vehicle in \cite{Laumond1994}.
\begin{lemma}
The length of the Reeds-Shepp path between two oriented points $(0,0,\Delta \theta)$ and $(0,0,0)$ is at most $3\rho \Delta\theta$.
\end{lemma}
\begin{proof}
Consider a feasible path for the Reeds-Shepp vehicle between $(0,0,\Delta \theta)$ and $(0,0,0)$ passing through points $A$,$B$ and $C$ as shown in Fig. \ref{fig:3rhotheta}. The length $dist_{rs} (\Delta \theta)$ of this path is equal to the sum of length of the four segments $OA$,$AB$,$BC$ and $CO$. From Fig. \ref{fig:3rhotheta}, one can deduce that $\alpha = \frac{\Delta\theta}{2}$. The length of the segment joining points $B$ and $C$ is denoted by $\overline{BC}$.
\begin{align*}
\centering
dist_{rs} (\Delta \theta) &= \rho \Delta\theta + \rho \alpha + \overline{BC} + \rho \alpha \\
& = \rho \Delta \theta + 2 \rho \alpha + \sqrt{ (\rho \sin \Delta\theta)^2 +(\rho- \rho \cos \Delta\theta )^2} \\
&= 2 \rho \Delta\theta + 2 \rho \sin \left( \frac{\Delta\theta}{2} \right) \\
& \le 3\rho \Delta\theta.
\end{align*}

\end{proof}
\end{document}